\documentclass{article}
\usepackage{amsmath}
\usepackage{amsthm}
\usepackage{graphicx}
\usepackage[font=small,labelfont=bf]{caption}
\DeclareCaptionType{copyrightbox}
\usepackage{subcaption}
\usepackage{amssymb}
\usepackage{amsfonts}
\usepackage{authblk}
\usepackage{url}
\usepackage{color}
\usepackage[lined, boxed]{algorithm2e}

\newcommand{\er}{Erd\H{o}s-R\'{e}nyi } 
\newtheorem{thm}{Theorem}

\begin{document}

\title{Locally Boosted Graph Aggregation for Community Detection \thanks{This work
is sponsored by the Assistant Secretary of Defense for Research \& Engineering
under Air Force Contract FA8721-05-C-0002.  Opinions, interpretations,
conclusions and recommendations are those of the authors and are not
necessarily endorsed by the United States Government.}}
\author[1,2]{Jeremy Kun}
\author[1]{Rajmonda S. Caceres}
\author[1]{Kevin M. Carter}
\affil[1]{MIT Lincoln Laboratory}
\affil[2]{University of Illinois at Chicago}

\date{}

\maketitle

\begin{abstract} \small \baselineskip=9pt 
Learning the right graph representation from noisy, multi-source data has
garnered significant interest in recent years. A central tenet of this problem
is relational learning. Here the objective is to incorporate the partial
information each data source gives us in a way that captures the true
underlying relationships. To address this challenge, we present a general,
boosting-inspired framework for combining weak evidence of entity associations
into a robust similarity metric. Building on previous work, we explore the
extent to which different local quality measurements yield graph
representations that are suitable for community detection. We present empirical
results on a variety of datasets demonstrating the utility of this framework,
especially with respect to real datasets where noise and scale present serious
challenges. Finally, we prove a convergence theorem in an ideal setting
and outline future research into other application domains.  \end{abstract}

\section{Introduction}
In the study of networks, the data used to define nodes and edges often come
from multiple sources. These sources are often noisy and ambiguously useful,
and the process of combining them into a single graph representation is
critically important. For example, suppose we are studying a social network and
wish to detect communities. The data that indicate membership in the same
community are plentiful: communication, physical proximity, reported
friendship, and many others. Each data source carries a different level of
information about the underlying social structure, and each may accurately
represent only some of the individuals. Some groups of friends communicate
through Facebook and others via Instagram, etc. The best way to aggregate 
this information is far from clear, and recent work has shown that the choice
of graph representation significantly impacts the performance of subsequent 
machine learning and data mining algorithms
\cite{Getoor2005,Gallagher2008,Neville2005,Caceres2011,Miller2014}. 

Further complicating matters, the quality of the aggregated graph depends
heavily on the application domain. Community detection is easiest with a graph
representation that retains only inter-community edges, but to predict the
spread of a virus one must have access to the strongest cross-community
conduits. Suitable graph representations for these two tasks may come from the
same data but are qualitatively different enough to warrant different
aggregation techniques. 

Even though the impact of the graph representation on subsequent analysis has
been widely studied, there are few techniques for learning the right graph
representations. Aggregation is often ad-hoc in practice, making it difficult
even to compare algorithms within the same application domain using different
data sources. The need for rigorous approaches to graph representation learning
is even more apparent with big data, where variety and veracity compound the
challenges of volume and velocity.

In this paper, we present a graph aggregation framework designed to make the
process of learning a useful underlying graph representation rigorous with respect
to application specific requirements. Our framework is called \emph{Locally
Boosted Graph Aggregation (LBGA)}. LBGA extracts the application-specific
aspects of the learning objective as an event $A$ representing an operation on
the graph (e.g. a clustering algorithm, a random walk, etc.) and a local
quality measure $q$. The framework incorporates this information into a
reward system that promotes the presence of good edges and the absence of bad
edges, in a fashion inspired by boosting.

Building on our work in~\cite{CCK14}, we demonstrate LBGA with the application
of community detection. In this context the goal of graph representation
learning is to aggregate the different data sources into a single graph which
makes the true community structure easy to detect.  LBGA evaluates the graph
data locally, so that it can choose the data sources which most accurately
represent the local structure of communities observed in real
networks~\cite{Aggarwal2011,Leskovec2008}. In the absence of ground truth
knowledge or one efficiently computable measure that can capture true community
quality, LBGA relies on the pair of a graph clustering algorithm $A$ and a
local clustering metric $q$ as an evaluation proxy.  We show through empirical
analysis that our algorithm can learn a high-quality global representation
guided by the local quality measures considered. 

We make the following contributions:

\begin{enumerate} 
   \item We present a graph aggregation framework that learns a useful graph
representation with respect to an application requiring only a local heuristic
measure of quality.
   \item Our framework is stochastic and incorporates both edge and non-edge
information, making it robust and suitable for sparse and noisy networks.
   \item We demonstrate the success of an algorithm implementing the framework
for community detection, testing it against both synthetic data and real-world
data sets. 
   \item We perform sensitivity and scalability analyses of our algorithm,
showing that the algorithm scales linearly in the number of edges and is
robust enough to handle large, noisy graphs.  
   \item We prove a convergence theorem for our framework and suggest the next
steps in proving performance guarantees.  
\end{enumerate} 

The rest of the paper is organized as follows. In Section~\ref{sec:related} we
give a brief overview of related literature. In Section~\ref{sec:lbga} we
discuss in detail the LBGA framework. In Section~\ref{sec:experiments} we
present the experimental analysis and results. In
Section~\ref{sec:additional-analysis} we discuss sensitivity to noise and
scalability. In Section~\ref{sec:convergence-theorem} we prove our convergence
theorem, and in Section~\ref{sec:conclusion} we discuss future work.

\section{Related Work} 
\label{sec:related}
\subsection{Representation Learning and Clustering}
Representation learning has garnered a lot of interest and research in recent
years. Its goal is to introduce more rigor to the often ad-hoc practices of
transforming raw, noisy, multi-source data into inputs for data mining and
machine learning algorithms. Within this area, representation learning of
graph-based data includes modeling decisions about the nodes of the graph, the
edges, as well as the critical features that characterize them both.

In this context, Rossi et al.~\cite{Rossi2012} discuss transformations to
heterogeneous graphs (graphs with multiple node types and/or multiple edge
types) in order to improve the quality of a learning algorithm. Within their
taxonomy, our work falls under the link interpretation and link re-weighting
algorithms \cite{Xiang2010,Gilbert2009}. Our setting is different because we
explicitly allow different edge types between the same pair of vertices. Also,
our approach is stochastic, which we find necessary for learning a robust
representation and weeding out noise. 

Clustering in multi-edge graphs
\cite{Papalexakis2013,Tang2009,Tang2012,Mucha2010,Berlingerio2011} is another
area with close connections to our work. A common thread among these existing
approaches is clustering by leveraging shared information across different
graph representations of the same data. These approaches do not address
scenarios where the information provided by the different sources is
complementary or the overlap is scarce. In contrast, our approach iteratively
selects those edge sources that lead to better clustering quality,
independently of disagreement across the different features.
\cite{Rocklin2013,Cai2005} present approaches for identifying the right graph
aggregation, given a complete ground truth clustering, or a portion of it
(i.e.: the cluster assignment is known only for a subset of the vertices in the
graph). Our framework requires no such knowledge, but we do use ground truth to
validate our experiments on synthetic data (Section \ref{sec:validation}). 

Balcan and Blum define in \cite{Balcan2006,Balcan2008} a list of intuitive
properties a similarity function needs to have in order to be able to cluster
well. However, testing whether a similarity function has the discussed
properties is NP-hard, and often dependent on having ground truth available.
Our model instead uses an efficiently computable heuristic as a rough guide.

\subsection{Boosting and Bandits}
Our framework departs from previous work most visibly through its algorithmic
inspirations, namely boosting \cite{Schapire90} and bandit learning (see
\cite{Bubeck12} for a survey of the latter). In boosting, we assume we have a
{\em weak classifier}, whose performance is only slightly better than random.
In his landmark paper \cite{Schapire90}, Schapire proved that weak classifiers
can be combined via a majority voting scheme to form a learner that achieves
arbitrarily close to perfect generalization. We think of different graph data
sources as weak learners in that they offer knowledge on when an edge should be
present. Then the question becomes whether one can ``boost'' the knowledge in
the different graphs to make one graph representation that is arbitrarily good.

Unfortunately, our problem setting does not allow pure boosting for two
reasons. First, boosting assumes the learners are all slightly better than
random, but graph representations can be pure noise or can even provide
\emph{bad} advice. Second, boosting has access to ground truth classification
labels. Even with reliable input data, the quality of the aggregation depends
on the application and many applications have no standard measure of quality. 

Ideas from bandit learning compensate for these problems. In bandit learning an
algorithm receives rewards as it explores a set of actions, and the goal is to
minimize some notion of regret in hindsight. The basic model has many variants,
but two central extensions in the literature are expert advice and adversaries.
Expert advice consists of functions suggesting to the algorithm what action to
take in each round (e.g., weak classifiers). The adversarial setting involves
an adversary who knows everything but the random choices made by the algorithm
in advance, and sets the experts or rewards so as to incur maximum regret. 

We apply these ideas to graph representation learning by setting up an
articifial reward system based on the given application and using stochasticity
to weed out adversarial portions of the input graphs. In our setting we only
care if the graph representation is good at the end, while bandit learning
often seeks to maximize cumulative rewards during learning. There are bandit
settings that only care about the final result (e.g., the pure exploration
model of Bubeck et al. \cite{Bubeck09}), but to the best of our knowledge no
theoretical results in the bandit literature immediately apply to our
framework. This is largely because we rely on heuristic proxies to measure the
quality of a graph, so even if the bandit learning objective is optimized we
cannot guarantee the result is useful.\footnote{For example, the empty graph
maximizes some proxies but is entirely useless.} Nevertheless we can adapt the
successful techniques and algorithms for boosting and bandit learning, and hope
they produce useful graphs in practice. As the rest of this paper demonstrates,
they do indeed.

The primary technique we adapt from bandits and boosting is the Multiplicative
Weights Update Algorithm (MWUA) \cite{Arora12}. The algorithm works as follows.
A list of weights is maintained on each element $x_j$ of a finite set $X$. At
each step of some process an element $x_i$ is chosen (in our case, by
normalizing the weights to a probability distribution and sampling), a reward
$q_{t,i}$ is received, and the weight for $x_i$ is multiplied or divided by $(1
+ \varepsilon q_{t,i})$, where $\varepsilon >0$ is a fixed parameter
controlling the rate of update. After many rounds, the elements with the
highest weight are deemed the best and used for whatever purpose needed.

\section{The Locally Boosted Graph Aggregation Framework}
\label{sec:lbga}

The Locally Boosted Graph Aggregation framework (LBGA) can succinctly be
described as running MWUA for each possible edge, forming a candidate graph
representation $G_t$ in each round by sampling from all edge distributions, and
computing local rewards on $G_t$ to update the weights for the next round. Over
time $G_t$ stabilizes and we produce it as output. The remainder of this
section expands the details of this sketch and our specific algorithm
implementing it.  

\subsection{Framework Details}
\label{sec:framework}

Let $H_1, \dots, H_m$ be a set of unweighted\footnote{There is a
natural extension for weighted graphs.}, undirected graphs defined on the same vertex set
$V$. We think of each $H_i$ as ``expert advice'' suggesting for any pair of
vertices $u,v \in V$ whether to include edge $e=(u,v)$ or not.  Our primary
goal is to combine the information present in the $H_i$ to produce a global
graph representation $G^*$ suitable for a given application. 

We present LBGA in the context of community detection, noting what aspects can
be generalized. Each round has four parts: producing the aggregate candidate
graph $G_t$, computing a clustering $A$ for use in measuring the quality of
$G_t$, computing the local quality of each edge, and using the quality values
to update the weights for the edges. After some number of rounds $T$, the
process ends and we produce $G^* = G_T$.

\textbf{Aggregated Candidate Graph $G_t$}: In each round produce a graph $G_t$
as follows. Maintain a non-negative weight $w_{u,v,i}$ for each graph $H_i$ and
each edge $(u,v)$ in $H_1 \cup \dots \cup H_m$. Normalize the set of all
weights for an edge $\mathbf{w}_{u,v}$ to a probability distribution over the
$H_i$; thus one can sample an $H_i$ proportionally to its weight. For each
edge, sample in this way and include the edge in $G_t$ if it is present in the
drawn $H_i$. 

\textbf{Event $A(G_t)$}: After the graph $G_t$ is produced, run a clustering
algorithm $A$ on it to produce a clustering $A(G_t)$. In this paper we fix $A$
to be the Walktrap algorithm \cite{Walktrap}, though we have observed the
effectiveness of other clustering algorithms as well. In general $A$ can be any
event, and in this case we tie it to the application by making it a simple
clustering algorithm.

\textbf{Local quality measure}: Define a \emph{local quality measure}
$q(G,e,c)$ to be a $[0,1]$-valued function of a graph $G$, an edge $e$ of $G$,
and a clustering $c$ of the vertices of G. The quality of $(u,v)$ in $G_t$ is
the ``reward'' for that edge, and it is used to update the weights of each
input graph $H_i$.  More precisely, the reward for $(u,v)$ in round $t$ is
$q(G_t, (u,v),A(G_t))$.

\textbf{Update Rule}: Update the weights using MWUA as follows. Define two
learning rate parameters $\varepsilon > 0, \nu > 0$, with the former being used
to update edges from $G_t$ that are present in $H_i$ and the latter for edges
not in $H_i$. In particular, suppose $q_{u,v}$ is the quality of the edge
$(u,v)$ in $G_t$. Then, the update rule is defined as follows:
\[
w_{u,v,i}=
\begin{cases}
w_{u,v,i}(1 +\varepsilon q_{u,v}), & \text{if } (u,v) \in H_i \\
w_{u,v,i}(1 - \nu q_{u,v}), & \text{if } (u,v) \not \in H_i .
\end{cases}
\]
 
\subsection{Quality Measures for Community Detection}
\label{sec:quality-measures}
We presently describe the two local quality measures we use for community
detection. The first, which we call {\em Edge Consistency} ($EC$) captures the
notion that edges with endpoints in the same cluster are superior to edges
across clusters:

\[
   EC_{u,v}=
   \begin{cases}
   1, & \text{if  }c(u) = c(v) \\
   0,  & \text{if  }c(u) \neq c(v).
   \end{cases}
\]
$EC$ offers a quality metric that is inextricably tied to the performance of
the chosen clustering algorithm.  The idea behind edge consistency can also be
combined with any quality function $q$ to produce a ``consistent'' version of
$q$. Simply evaluate $q$ when the edge is within a cluster, and $-q$ when the
edge is across clusters. Note that $q$ need not depend on a clustering of the
graph or the clustering algorithm, and it can represent algorithmic-agnostic
measures of clustering quality.

As an example of such a measure $q$, we consider the metric of
\emph{Neighborhood Overlap} ($NO$), which uses the idea that vertices that
share many neighbors are likely to be in the same community. NO declares that
the quality of $(u,v)$ is equal to the (normalized) cardinality of the
intersection of the neighborhoods of $u$ and $v$: 

\[ 
   NO_{u,v}=\frac{|N(u) \cap N(v)|}{|N(u) \cap N(v)| + log(|V|)}, 
\] 
where $N(x)$ represents the neighborhood of vertex $x$. We have also run
experiments using more conventional normalizing mechanisms, such as the Dice
and Jaccard indices~\cite{Dice1945,Jaccard1912}), but our neighborhood overlap
metric outperforms them by at least 10\% in our experiments. We argue this is
due to the use of a global normalization factor, as opposed to a local one,
which is what Dice and Jaccard indices use. 
For brevity and simplicity, we omit our results for Jaccard and Dice indices and
focus on Neighborhood Overlap. In our experimental analysis
(Section~\ref{sec:results}) we use the consistent version of $NO$, which we
denote \emph{consistentNO}. 

While we demonstrate the utility of the LBGA framework by using $EC$ and
$consistentNO$, the design of the framework is modular, in that the mechanism
for rewarding the ``right'' edges is independent from the definition of reward.
This allows us to plug in other quality metrics to guide the graph
representation learning process for other applications, a key goal in LBGA's
design.

\subsection{LBGA Implementation} 
Processing every edge in every round of the LBGA framework is inefficient.
Our implementation of LGBA, given by Algorithm~\ref{alg:nef}, improves
efficiency by fixing edges whose weights have grown so extreme so as to be
picked with overwhelming or negligible probability (with probability $ >
1-\delta$ or $< \delta$ for a new parameter $\delta$). In practice this
produces a dramatic speedup on the total runtime of the
algorithm.\footnote{From days to minutes in our experiments.} The worst-case
time complexity is the same, but balancing parallelization and the learning
parameters suffices for practical applications. 

In addition, our decision to penalize non-edges ($\nu > 0$) also improves
runtime from the alternative ($\nu = 0)$. In our experiments non-edge feedback
causes $G_t$ to convergence in roughly half as many rounds as when only
presence of edge is considered as indication of relational structure.

We also note that Algorithm \ref{alg:nef} stays inside the ``boundaries'' 
determined by the input graphs $H_i$.  It never considers edges that are not
suggested by \emph{some} $H_i$, nor does it reject an edge suggest by all 
$H_i$. Thus, when we discuss sparsity of our algorithm's output in our
experiments, we mean with respect to the number of edges in the union of the
input graphs.

\begin{algorithm}[tbh]
\caption{Optimized implementation of LBGA. Note that $1_E$ denotes the
characteristic function of the event $E$.}
\label{alg:nef}
   \DontPrintSemicolon
   \SetAlgoLined
   {\footnotesize
   \KwData{Unweighted graphs $H_1, \dots, H_m$ on the same vertex set $V$, a
clustering algorithm $A$, a local quality metric $q$, three parameters
$\varepsilon, \nu, \delta > 0$}
   \KwResult{A graph $G$}
   Initialize a vector $\mathbf{w}_{u,v} = \mathbf{1}$ for all $u \neq v \in V$\;
   Let $U$ be the edge set of $H_1 \cup \dots \cup H_m$\;
   Let $G_\textup{learned} = (V, \varnothing)$ \;
   \While{$|U| > 0$}{
      Let $G$ be a copy of $G_{\textup{learned}}$\;

      \For{$(u,v) \in U$}{
         Let $p_{u,v} = \frac{\sum_i w_{u,v,i} 1_{\left \{(u,v) \in H_i \right \}}}{\sum_i w_{u,v,i}}$ \;
         Flip a coin with bias $p_{u,v}$\;
         If heads, include $(u,v)$ in $G$.
      }

      Cluster $G$ using $A$\;

      \For{$(u,v) \in U$}{
         Set $p = q(G, A(G), (u,v))$\;
         \For{$i = 1, \dots, m$}{
            \eIf{$(u,v) \in H_i$}{
               Set $w_{u,v,i} = w_{u,v,i} (1 + \varepsilon p)$\;
            } {
               Set $w_{u,v,i} = w_{u,v,i} (1 - \nu p)$\;
            }
         }

         Let $p_{u,v} = \frac{\sum_i w_{u,v,i} 1_{\left \{(u,v) \in H_i \right \}}}{\sum_i w_{u,v,i}}$ \;
         \If{$p_{u,v} > 1-\delta$}{
            Add $(u,v)$ to $G_{\textup{learned}}$, remove it from $U$\;
         }
         \If{$p_{u,v} < \delta$}{
            Remove $(u,v)$ from $U$\;
         }
      }
   }
   Output $G$\;
}
\end{algorithm}

\section{Experimental Analysis}
\label{sec:experiments}
We presently describe the datasets used for analysis and provide quantitative
results for the performance of Algorithm \ref{alg:nef}. 

\subsection{Synthetic Datasets}
\label{sec:synthetic-model}

Our primary synthetic data model is the stochastic block model \cite{Wang87},
commonly used to model explicit community structure.  We construct a
probability distribution $G(\mathbf{n},B)$ over graphs as follows. Given a
number $n$ of vertices and a list of cluster (block) sizes $\mathbf{n}=\{n_1,
\dots, n_k\}$ such that $n =\sum_i n_i$, we partition the $n$ vertices into $k$
blocks $\{b_1, \dots, b_k\}$, $|b_i|=n_i$.  We declare that the probability of
an edge occurring between a vertex in block $b_i$ and block $b_j$ is given by
the $(i,j)$ entry of a $k$-by-$k$ matrix $B$. In order to simulate different
scenarios, we consider the following three cases.
 
{\em Global Stochastic Block Model (GSBM):} In this model we have $m$ input
graphs ${H_i,\ldots,H_m}$, each drawn from the stochastic block model
$G(\mathbf{n}, B_i)$ \footnote{$G(\mathbf{n}, B_i)$ represents a simpler case
of the stochastic block model, where the within-cluster probabilities are
uniform across blocks and blocks have the same size.}, with $n_1 = \dots = n_m$
and $B_i$ defined as:

\[
   B_i = 
   \begin{pmatrix}  
      p_i        &  r_i        &  r_i        & \dots  &  r_i\\
      r_i        &  p_i        &  r_i        & \dots  &  r_i \\
      \vdots     & \vdots      & \vdots      & \ddots &  \vdots      \\
      r_i        &  r_i        & r_i         & \dots  &  p_i
   \end{pmatrix},
\]
where $p_i$ represents the within-cluster edge probability and $r_i$ represents
the across-cluster edge probability in graph $H_i$. The ratio $SNR=p_i/r_i$ is
commonly referred to as the {\em signal to noise} ratio and captures the
strength of community structure within $H_i$. We use the GSBM case to model a
scenario where each graph source has a global (or uniform) contribution toward
the quality of the targeted graph representation $G^*$.

{\em Local Stochastic Block Model (LSBM):} This scenario captures the notion
that one graph source accurately describes one community, while another source
fares better for a different community. For example, if we have two underlying
communities, and two graph sources $H_1, H_2$, then we use the following two
block matrices to represent them:

\[
B_1=\begin{pmatrix}
p & r \\
r & r
\end{pmatrix},
\hspace{0.5cm}
B_2=\begin{pmatrix}
r & r \\
r & p
\end{pmatrix}
.\]

This naturally extends to a general formulation of the LSBM model for $m$
communities.

{\em \er (ER) model:} Finally, we consider the case of the \er random
graph~\cite{Erdos60}, where any two vertices have equal probability of being
connected. This model provides an example of a graph with no community
structure. Note that the ER model is a special case of both GSBM and LSBM with
$p=r$. In our experimental analysis we consider cases where an ER model is
injected into instances of GSBM and LSBM in order to capture a range of
structure and noise combinations.

\begin{table*}
\centering
\resizebox{0.7\textwidth}{!}{
\begin{tabular}{| l | l |}
\hline
Dataset & Parameters \\
\hline \hline
GSBM-1& $m=k=4$,$n_i=125$, $p_i=0.2$, $r_i=0.05$, $i=1, \ldots,m$ \\
GSBM-2& $m=k=4$, $n_i=125$, $p_i=0.3$, $r_i=0.05$, $i=1, \ldots,m$ \\
GSBM-3& $m=5, k=4$, $n_i=125$, $p_i=0.3$, $r_i=0.05$, $i=1, \ldots,4$, $p_5 = r_5 = 0.01$ \\
\hline 
GSBM-4& $m =k= 4$, $n_i=125$,$p_1=0.1625,p_2 = 0.125,p_3 = 0.125,p_4 = 0.0875$,$r_i = 0.05$, $i=1, \ldots,m$\\
GSBM-5 & $m=k=4$, $n_i=125$,$p_1=0.15,p_2=0.1,p_3=p_4=0.05$, $r_i = 0.05$, $i=1, \ldots,m$\\

\hline
LSBM-1& $m=k=4$, $n_i=125$, $p_i=0.2$, $r_i=0.05$, $i = 1, \dots, m$ \\
LSBM-2 & $m=k=4$, $n_i=125$, $p_i=0.3$, $r_i=0.05$, $i = 1, \dots, m$ \\
LSBM-3& $m=5,k=4$, $n_i=125$, $p_i=0.3$, $r_i=0.05$, $i = 1, \dots, m$, $p_5= r_5 = 0.01$ \\
\hline
ER only &  $m=4$, $p_i=r_i=0.01$ \\
DBLP & $n = 3153, m = 2$ \\
RealityMining & $n = 90, m = 6$ \\
Enron & $n = 145, m = 2, \alpha=0.9$ \\
\hline
\end{tabular}
}
\caption{Description of datasets analyzed. Total number of vertices in each
synthetic source graph is $n=500$.  $m$ is the number of graph sources. $k$ is the number of clusters. $n_i$ represents
number of vertices in cluster $i$. $p_i$ and $r_i$ represent the within- and
across-cluster edge probability for each the $m$ graph sources.}
\label{datasets}
\end{table*}

\subsection{Real Datasets} 
\subsubsection{DBLP}
Our first real-world dataset is DBLP \cite{Ley02}, a comprehensive online
database documenting research in computer science. We extracted the subset of
the DBLP database corresponding to researchers who have published at two
conferences: the Symposium on the Theory of Computing (STOC), and the Symposium
on Foundations of Computer Science (FOCS). The breadth of topics presented at
these conferences implies a natural community structure organized by sub-field.
Each node in the DBLP graph represents an author, and we use two graphs on this
vertex set: the {\em co-authorship} graph and the {\em title similarity} graph.
For the latter, we add an edge between two author vertices if any of their
paper titles contain at least three words in common (excluding stop words), and
the weight of this edge is the number of such pairs of papers. We considered a
total of 5234 papers across 3153 researchers.

\subsubsection{RealityMining}
Our second dataset is RealityMining \cite{RealityMining}, a 9-month experiment
in 2004 which tracked a group of 90 individuals at MIT via sensors in their
cell phones. The individuals were either associated with the MIT Media Lab or
the Sloan Business School, and there is a natural corresponding community
structure. The data collected include voice calls, bluetooth scan events at
five-minute intervals, cell tower usage, and self-reported friendship and
proximity data. The data set is naturally noisy: surveys are subjective
estimates, cell tower ranges are only so precise and signal outages are common,
and there was data loss from typical cell phone problems like running out of
battery. 

We used the subset of subjects participating between 2004-09-23 19:00:00 and
2005-01-07 18:00:00 (UTC-05:00), for a total of 3354 call events, 786301 cell
tower transition events, and 689025 bluetooth scan events. The nodes in our
graphs represent individuals in the study. Weighted edges correspond to the
total duration of voice calls, the total amount of time two individuals used
the same cell tower, the total number of bluetooth events, and the results of
the friendship/proximity surveys for a total of six graphs. Our results were
stable under somewhat drastic changes in the input graphs (for example, in
ignoring weights).

\subsubsection{Enron}
Our final dataset is the Enron email dataset~\cite{EnronConf, Enron}, a
well-studied corpus of over 600,000 emails sent between 145 employees of the
Enron Corporation in the early 2000's. We produced two graphs from the Enron
data, one for peer-to-peer email communication and one for topic similarity in
the email content. 

In both graphs the vertices are individuals. In the email graph, the edges are
weighted by the number of emails sent between the individuals in question. We
used the Mallet package~\cite{mallet} to generate the LDA topic model for the
content of Enron email data. We aggregated into one document all the email
content sent by each of the Enron employes considered in the email link graph.
Each document and therefore each sender is represented by 60 topics. We measure
cosine distance of the topic vectors of individuals, and considered an edge as
present if the cosine distance was above a specified threshold value
$\alpha$.\footnote{We experimented with various threshold values, and we
discuss this in Section~\ref{sec:enron-results}.}

Table~\ref{datasets} contains a summary of all the datasets used for
the experimental analysis and their parameters. 

\subsection{Validation Procedure} 
\label{sec:validation}
In our work, the optimality of the graph representation is closely coupled with
the quality of community structure captured by the representation. This gives
us several ways of evaluating the quality of the results produced by our
algorithm. We consider notions of quality reflected at different levels: the
quality of cluster assignment, the quality of graph representation, and the
quality of graph source weighting. 

{\em Quality of Cluster Assignment:} Since the output of LBGA is a graph, we
use the walktrap clustering algorithm to extract communities for analysis. We
then compare these communities to the ground truth clustering, when it is
available, or else to the known features of the datasets.  We use the
Normalized Mutual Information (NMI) measure \cite{Danon05} to capture how well
the ground truth clustering overlaps with the clustering on the graph
representation output from our algorithm. In general we find that the choice of
clustering algorithm is unimportant, because the graph output by LBGA is
sufficiently modular to admit only one reasonable clustering. Walktrap is
further convenient in that we need not assume the number of clusters ahead of
time.

{\em Quality of Graph Representation:} An ideal graph representation that
contains community structure would consist of disjoint cliques or near-cliques
corresponding to the communities. As we illustrate in
Section~\ref{sec:results}, an optimal graph representation can do better than
just produce a perfect clustering. It can also remove cross-community edges and
produce a sparser representation, which is what our algorithm does. We use two
measures of clusterness to capture this notion of graph representation quality.
Modularity \cite{Newman06} is a popular measure that compares a given graph and
clustering to a null model. Conductance~\cite{Leskovec2008,Gleich2012} measures
relative sparsity of cuts in the graph. Since conductance is defined for a
single cut, we compute it for a clustering as the sum of the conductance values
of cuts defined by isolating a single cluster from the rest of the graph. Note
that \emph{higher} modularity scores and \emph{lower} conductance scores
signify stronger community structure Both modularity and conductance are
well-known and often offer complimentary information about the quality of
communities.
 
We note two extreme graph representation cases, the empty graph which is
perfectly modular in a degenerate sense, and the union graph which is a trivial
aggregation. To signal these cases in our results, we display the
\emph{sparsity} of the produced graph $G^*$, defined as the fraction of edges
in $G^*$ out of the total set of edges in all input graphs. 

{\em Quality of Graph Source Weighting:} the quality of the aggregation process
is captured by the right weighting of individual edge sources. Edge sources (input
graphs) that are more influential in uncovering the underlying community
structure have higher weights on average. Similarly, edge types that contribute
equally should have equal weights, and edge types with no underlying structure
should have low weights.

\begin{figure}[t]
\begin{centering}
\includegraphics[width=\columnwidth]{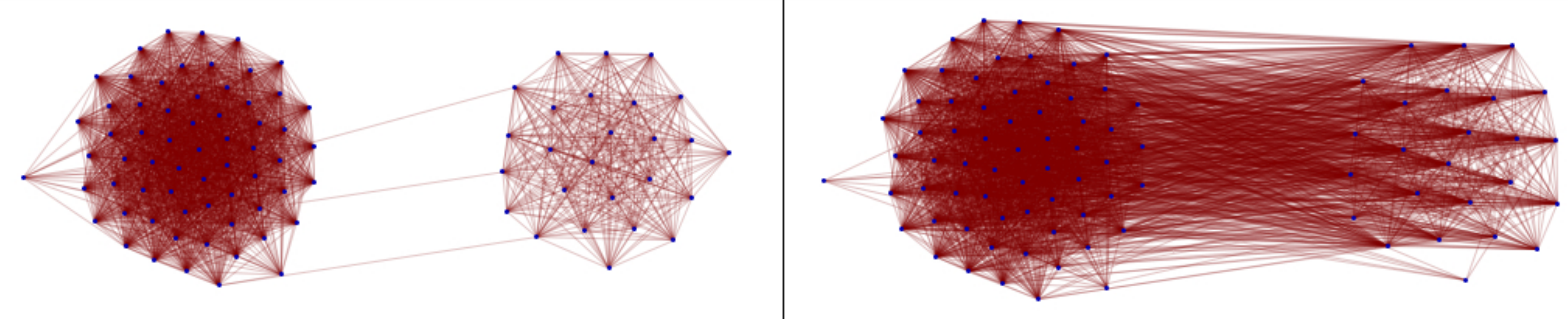}
\par\end{centering}
\caption{Left: the results of LBGA on the RealityMining dataset. Right: the
input graph of Bluetooth scan events. LBGA was run with $consistentNO$, $\nu =
\varepsilon = 0.2$, $\delta = 0.05$} 
\label{fig:reality-mining-comparison}
\end{figure}

\begin{figure}[t]
\begin{centering}
\includegraphics[width=\columnwidth]{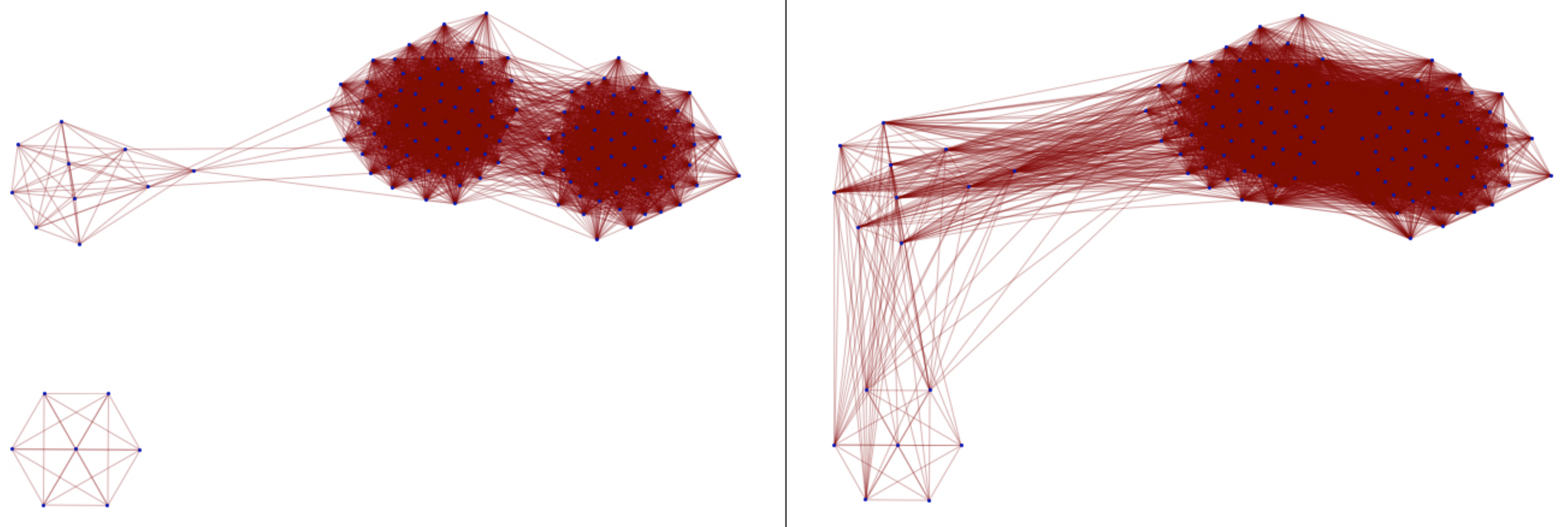}
\par\end{centering}
\caption{Left: the results of LBGA on the Enron dataset. Right: the input graph
of topic models thresholded at 0.9. LBGA was run with $consistentNO$, $\nu =
\varepsilon = 0.2$, $\delta = 0.05$} 
\label{fig:enron-comparison}
\end{figure}

\subsection{Experimental Results}
\label{sec:results}

Table~\ref{EC_NO} contains the numerical results of our experiments. As a
baseline, we computed the modularity values of the union of the input
graphs with respect to the ground truth (for synthetic) or the Walktrap
clusterings (for the real world). We now discuss the specific results for the
synthetic and real data sets. 

\subsubsection{Synthetic}
For illustration, we show in Figure \ref{fig:local-sbm} the performance of
Algorithm \ref{alg:nef} when $consistentNO$ is used as a local quality metric
and LSBM-3 (see Table \ref{datasets} for details) is used to generate the input
graphs. Note that the algorithm converges quickly to a graph which results in a
perfect clustering as measured by NMI. We also plot the modularity of the
resulting graph produced in each round, seeing that it far exceeds the
``baseline'' modularity of the union of the input graphs. This tells us the
learning algorithm is able to discard the noisy edges in the model. Finally, we
plot the number of edges in the graph produced in each round, and the average
vertex-pair weight for each input graph. This verifies that our algorithm
complies with our edge-type weighting and sparsity requirements. Indeed, the
algorithm produces a relatively sparse graph, using about 40\% of the total
edges available and weights edges from the \er source appropriately.  Our
algorithm hence achieves a superior graph than the union, while preserving the
underlying community structure so as to be amenable to clustering. 

The results for the other synthetic datasets are similar and summarized in
Table~\ref{EC_NO}. We note that our algorithm's performance degrades when the
noise becomes too high. In Section~\ref{sec:sensitivity-analysis} we analyze
the signal to noise ratio in the synthetic data sets more closely.

\subsubsection{DBLP}
In Figure~\ref{fig:dblp}, we show results for the DBLP dataset. For both edge
consistency and $consistentNO$, our algorithm converges to a graph of modularity
exceeding that of the union graph and using significantly fewer edges (60\% in
the case of $consistentNO$ and 88\% for edge consistency).

Our algorithm selects title similarity as having more influence in recovering
communities for the STOC/FOCS conferences. Researchers attending these
conferences represent a small community as a whole with many of them sharing
co-authorship on papers with diverse topics. In this sense, it is not
surprising that title similarity serves as a better proxy for capturing the
more pronounced division along topics. We have also manually inspected the
resulting clusters, and most appear organized both by membership and
coauthorship. Take, for example, Mikko Koivisto, Thore Husfeldt, Petteri Kaski,
and Andreas Bj\"{o}rklund. They have together coauthored over 15 papers in
combinatorial optimization, and naturally fall within the same small
coauthorship cluster. However, the title similarity graph alone yields around
1500 clusters, and these researchers are split across two clusters because of
the differences in their non-coauthored work. They fall in the same cluster in
the LBGA-aggregated graph, and the cluster is larger, including well-known
researchers who are either coauthors with some of the four or have done much
work in the same field. Though this is a promising sign, we suggest a more
thorough quantitative analysis of the quality of this clustering for future
work.

\subsubsection{RealityMining}
Much work has been done in manually constructing good graph representations for
the RealityMining social network (e.g. aggregating Bluetooth via thresholding
and picking useful time windows \cite{RealityMining,Caceres2011}).  LBGA,
however, arrives at an equally good representation from the raw, noisy input
graphs. The final graph it constructs contains two dense clusters corresponding
exactly to the MIT Media Lab and the Sloan Business School, with only three
edges crossing the cut. The final group uses 63.5\% of the total edges
available and has modularity 0.25 with respect to the Media/Sloan partition.
See Figure \ref{fig:reality-mining-comparison} for a comparison of the original
Bluetooth graph and the final produced graph.

\subsubsection{Enron}
\label{sec:enron-results}
The data of Table~\ref{EC_NO} shows that $consistentNO$ achieves a graph
representation with higher modularity, lower conductance and better sparsity
when compared to the edge consistency measure.
Figure~\ref{fig:enron-comparison} shows a clear community structure, and there
the smaller clusters correspond to lower-level employees while the higher level
managers reside in the bigger clusters. Additionally, there was a known fantasy
sports community within the network, and all of these individuals fall within a
single cluster in the graph output by LBGA~\cite{Mccallum05}. 

We also investigated the effect of changing the threshold value $\alpha$ for
considering an edge in the topic graph. A value of less than $\alpha = 0.7$
always produced two large dense clusters with many noisy edges between them
(akin to the two large clusters in Figure~\ref{fig:enron-comparison}). In our
experiments we used $\alpha = 0.9$, although values of $\alpha$ as high as
$0.95$ gave qualitatively similar results.

\subsubsection{Comments}
Overall, we find that LBGA converges to graphs of both high modularity and low
conductance score. It also generates graph representations that induce correct
clusterings in almost all cases where some sort of ground truth is known, the
challenging case being when SNR is low.  Moreover the algorithm weights the
different input graphs appropriately to their usefulness. We find that the edge
consistency measure outperforms neighborhood overlap in terms of overlap with
ground truth clustering when the signal to noise ratio is particularly low, but
that in almost all other cases (especially the real data sets), $consistentNO$
produces higher quality, sparser, and more modular graphs.

We notice that the modularity values for RealityMining and the Enron data set
are significantly smaller after passing through LBGA when compared to the
baseline union values. We argue that this is due to the relatively small
clusters produced by LBGA, as it is a known shortcoming that modularity is not
an accurate measure when the communities are small~\cite{Fortunato07}. Indeed,
when conductance is used as a clustering quality measure LBGA significantly
outperforms the baseline union aggregation.
Figure~\ref{fig:reality-mining-comparison} shows the favorable structure of the
RealityMining dataset after being passed through LBGA (and the noisy input
data), and Figure~\ref{fig:enron-comparison} gives a similar picture for the
Enron dataset. 

\begin{figure}[t]
\begin{centering}
\includegraphics[width=\columnwidth]{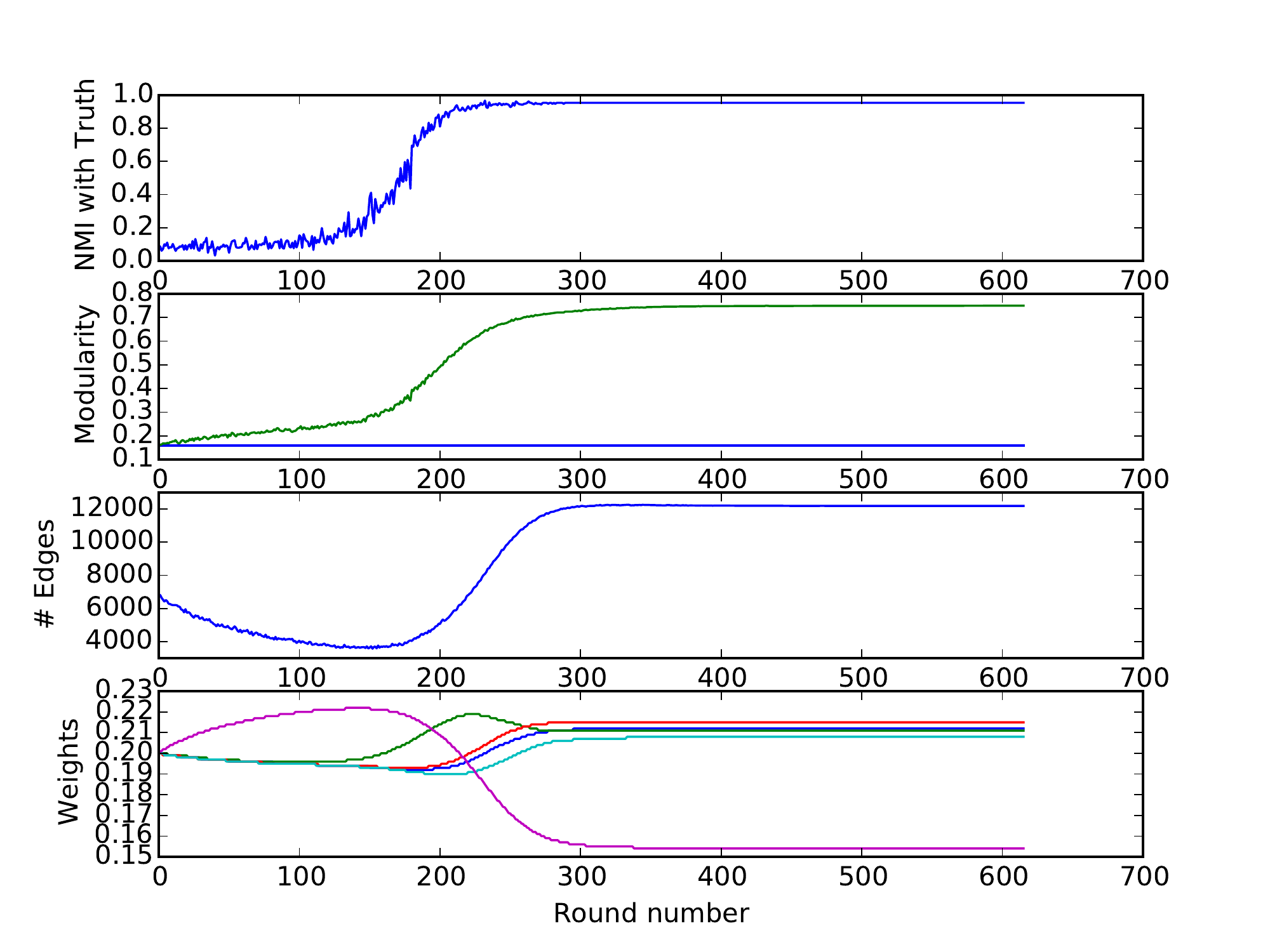}
\par\end{centering}
\caption{Graph representation learning for LSBM-3. The LBGA parameters are
$\varepsilon=\nu=0.2, \delta=0.05$. Plots in order top to bottom: 1. NMI of
$A(G_t)$ with the ground truth clustering, 2. modularity of $G_t$ w.r.t
$A(G_t)$, with the horizontal line showing the modularity of the union of the
input graphs w.r.t. ground truth, 3. the number of edges in $G_t$, 4.  the
average probability weight (quality) of vertex pairs for $H_i$.  The \er graph
converges to low weight by round 300.} 
\label{fig:local-sbm} 
\end{figure}

\begin{figure}[t]
\begin{centering}
\includegraphics[width=\columnwidth]{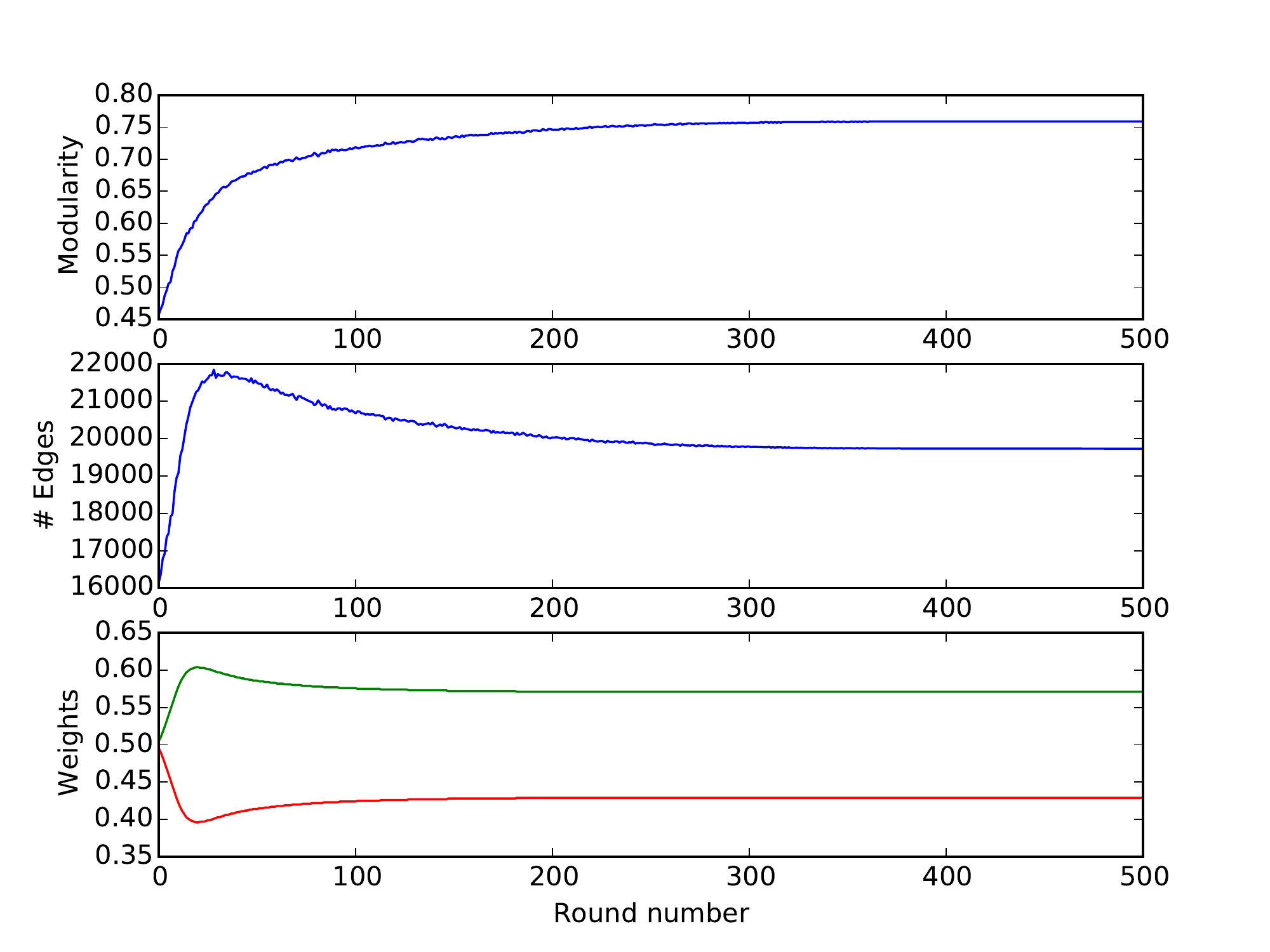}
\par\end{centering}
\caption{Aggregation of co-authorship (red curve) and title similarity graphs
(green curve) for DBLP dataset.} 
\label{fig:dblp}
\end{figure}

\begin{table*}
\centering
\resizebox{\textwidth}{!}{
\begin{tabular}{| l | c c  | c c c  r l | c c c  r l |}
\hline 
\multicolumn{1}{| l}{} &  \multicolumn{2}{|l|}{Union Graph} & \multicolumn{5}{c|}{EC} & \multicolumn{5}{c|}{ConsistentNO}\\
\hline
Dataset  & Modularity & Conductance &  Modularity & Conductance & NMI & Sparsity & Edge Type Weights  &  Modularity & Conductance & NMI & Sparsity & Edge Type Weights \\
\hline \hline
GSBM-1 & 0.264  & 7.568 &   0.549 & 2.042 &  1     &  0.644 &  (0.250,0.251,0.250,0.249) &    0.750  & 0.002 & 1     &  0.515    & (0.250,0.251,0.249,0.249)\\
GSBM-2 & 0.323   & 5.887 & 0.580 & 1.586 & 1     &  0.691 &  (0.252,0.250,0.248,0.251) &    0.750  & 0 & 1     &  0.573    & (0.252,0.250,0.247,0.251)\\ 
GSBM-3 & 0.312   & 6.140 & 0.607 & 1.333 & 1     &  0.657 &  (0.225,0.224,0.226,0.227,0.098) &   0.750 & 0 & 1     &  0.562    & (0.221,0.221,0.222,0.223,0.113)\\
\hline
GSBM-4 & 0.143 & 12.212 & 0.421 & 2.659 & 0.966 & 0.585 & (0.202,0.232,0.265,0.302)  & 0.750 & 0.001 & 0.983	 & 0.393 &  (0.202,0.231,0.266,0.302)\\
GSBM-5 & 0.145 & 6.584 & 0.395 & 2.124 & 0.919 & 0.653 & (0.213,0.282,0.361,0.144) & 0.666  & 0.001 & 0.958 & 0.477 & (0.199,0.271,0.348,0.182)\\
\hline
LSBM-1 & 0.111  & 14.288 &   0.298 & 4.316 &  0.765 &  0.651 &  (0.253,0.250,0.250,0.248) &    0.378  & 122.601 & 0.032 &  0.060     & (0.249,0.251,0.250,0.250)\\
LSBM-2 & 0.167  & 11.106 &   0.464 & 3.200 &  0.975     &  0.582 &  (0.249,0.251,0.248,0.252) &    0.750  & 0.001 & 1 &  0.417   & (0.250,0.250,0.248,0.252)\\
LSBM-3 & 0.162  & 11.701 &   0.473 & 3.009 &  0.966 &  0.568 &  (0.218,0.217,0.222,0.219,0.124) &    0.750  & 0.001 & 0.968 & 0.395    & (0.212,0.212,0.213,0.209,0.154)\\
\hline
ER only     & -0.002 & 24.729 &   0.193 & 112.947 &  0.012 &  0.999  &  (0.264,0.234,0.260,0.243) &    0.836  & 1.068 & 0.025 &  0.230     & (0.251,0.253,0.248,0.247)\\
\hline
DBLP   & 0.386 & 1368.859 &   0.372 & 1214.824 &  NA    &  0.962 &  (0.270,0.730) &    0.695  & 159.286 & NA    &  0.632    & (0.432,0.568)\\
RealityMining & 0.452 & 70.314 &   0.196 & 1.538 &  NA  &  0.724 &  (0.394,0.080,0.226,0.100,0.100,0.100) &  0.246 & 0 & NA    &  0.646    & (0.365,0.091,0.198,0.115,0.115,0.115) \\
Enron  & 0.559 & 134.572 &  0.190 & 11.092 &  NA    &  0.921 &  (0.193,0.807) &    0.444  & 0.594 & NA    &  0.631    & (0.390,0.610)\\
\hline
\end{tabular}
}
\caption{LBGA performance results. All datasets in this table were run with EC
and $consistentNO$ using $\varepsilon = \nu = 0.2, \delta = 0.05$. Union
modularity and conductance for real datasets was computed with the walktrap
clustering. The order of edge type weights for the real datasets are: DBLP
(coauthorship, title similarity); RealityMining (bluetooth, phone calls, cell
tower proximity, reported friendship, in-lab proximity, out-lab proximity);
Enron (email, topic similarity).} 
\label{EC_NO}
\end{table*}

\section{Robustness and Scalability}
\label{sec:additional-analysis}

\subsection{Sensitivity Analysis} \label{sec:sensitivity-analysis} We analyze
the sensitivity of LBGA to noise. In Figure~\ref{fig:sensitivity-analysis} we
display performance as measured by NMI when the graph inputs are LSBM models
across different intra-cluster probability values $p_i$ and varying SNR values.
We make the following general observations.  The algorithm is consistent in
that as the noise rate $r_i$ increases, the NMI values do deteriorate as
expected. The algorithm both reaches higher quality and maintains the quality
longer for denser graphs, which again is consistent with our expectations. It
seems the critical SNR (when the NMI drastically drops) is higher for
consistentNO than for EC. At a signal to noise ratio of 2 or less, the NMI is
bad for both quality measures and all choices of $p_i$. The sharp drop in
quality is related to the well-known  phase transition in the community
detection problem \cite{nadakuditi2012}.  Therefore, the distance from the
theoretical detectability bound on community detection could be used as an
agnostic measure to gauge the usefulness of LBGA with a particular quality
metric.

\begin{figure}[bht]

\begin{subfigure}{0.01\columnwidth}
     \includegraphics[scale=0.7]{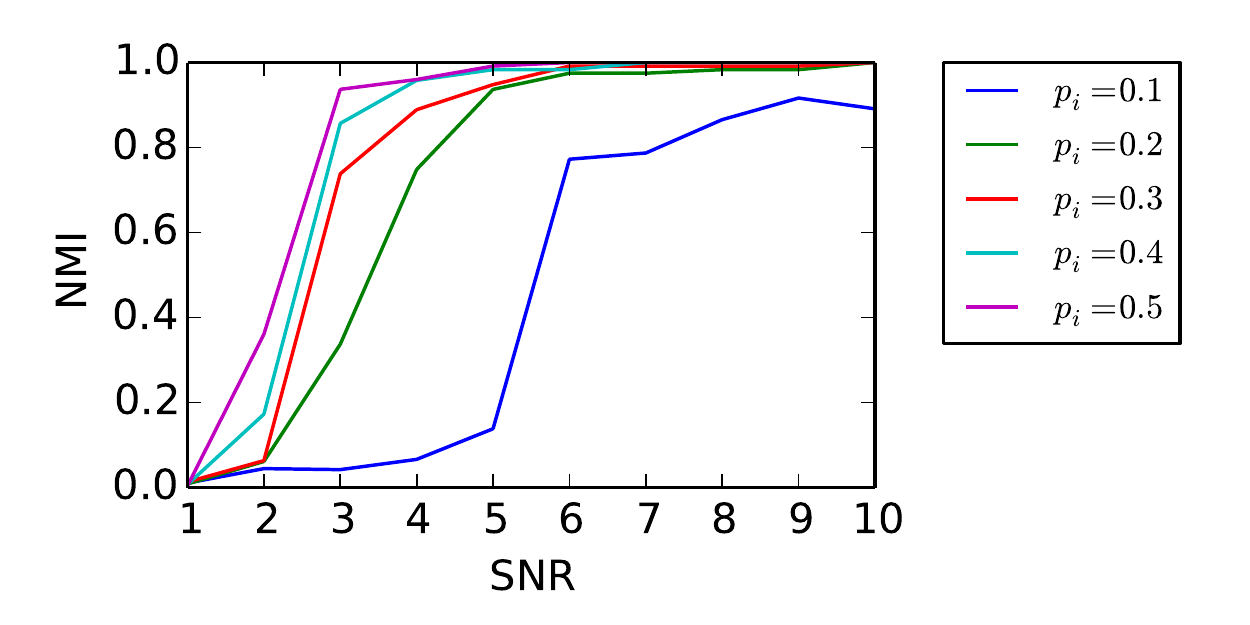}
\end{subfigure}    

\begin{subfigure}{0.01\columnwidth}
    \includegraphics[scale=0.7]{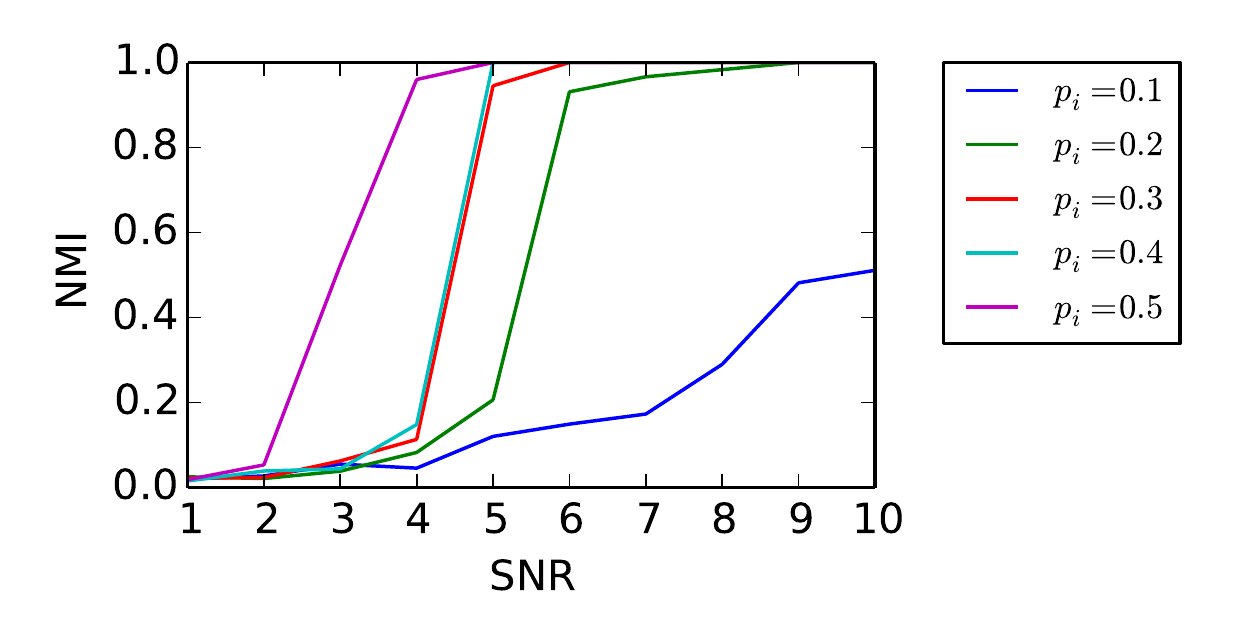}
\end{subfigure}

\caption{Performance of LBGA (measured by NMI) as a function of SNR for the
LSBM model with different probabilities $p_i$ for the edge consistency measure
(top) and $consistentNO$ (bottom).} 
\label{fig:sensitivity-analysis}
\end{figure}

\subsection{Scalability Analysis}
\label{sec:scalability-analysis}

We run a larger version of the LSBM-2 model ($m=k=10,p_i=0.3,r_i=0.05$) to
illustrate how LBGA scales for larger graphs. As shown in
Figure~\ref{fig:scalability-analysis}, LBGA scales linearly with the number of
edges. Given that real-world graphs are usually sparse, this scaling behavior
makes LBGA computationally suitable for large graphs.

LBGA was designed and implemented with scaling in mind, and the process of
fixing vertex-pairs as they are learned is largely the reason for the nice
scaling propertes of Algorithm~\ref{alg:nef}. Moreover, the parameter $\delta$
encapsulates a trade-off between runtime and accuracy. Should linear scaling be
insufficient, LBGA's design allows for additional modifications to improve
scalability.  For example: quality functions that are sufficiently local allow
one to parallelize the weight update step; one could sample a sublinear number
of edges in each round and only update weights within the subgraph; one might
specify a set of ``seeds'' (vertices of interest), and sample edges local to
those vertices. While these methods are currently just potential directions for
future work, there is clearly much potential for further improvements of LBGA's
computational efficiency.

\begin{figure}[bht]
\begin{centering}
\scalebox{0.7}{\includegraphics{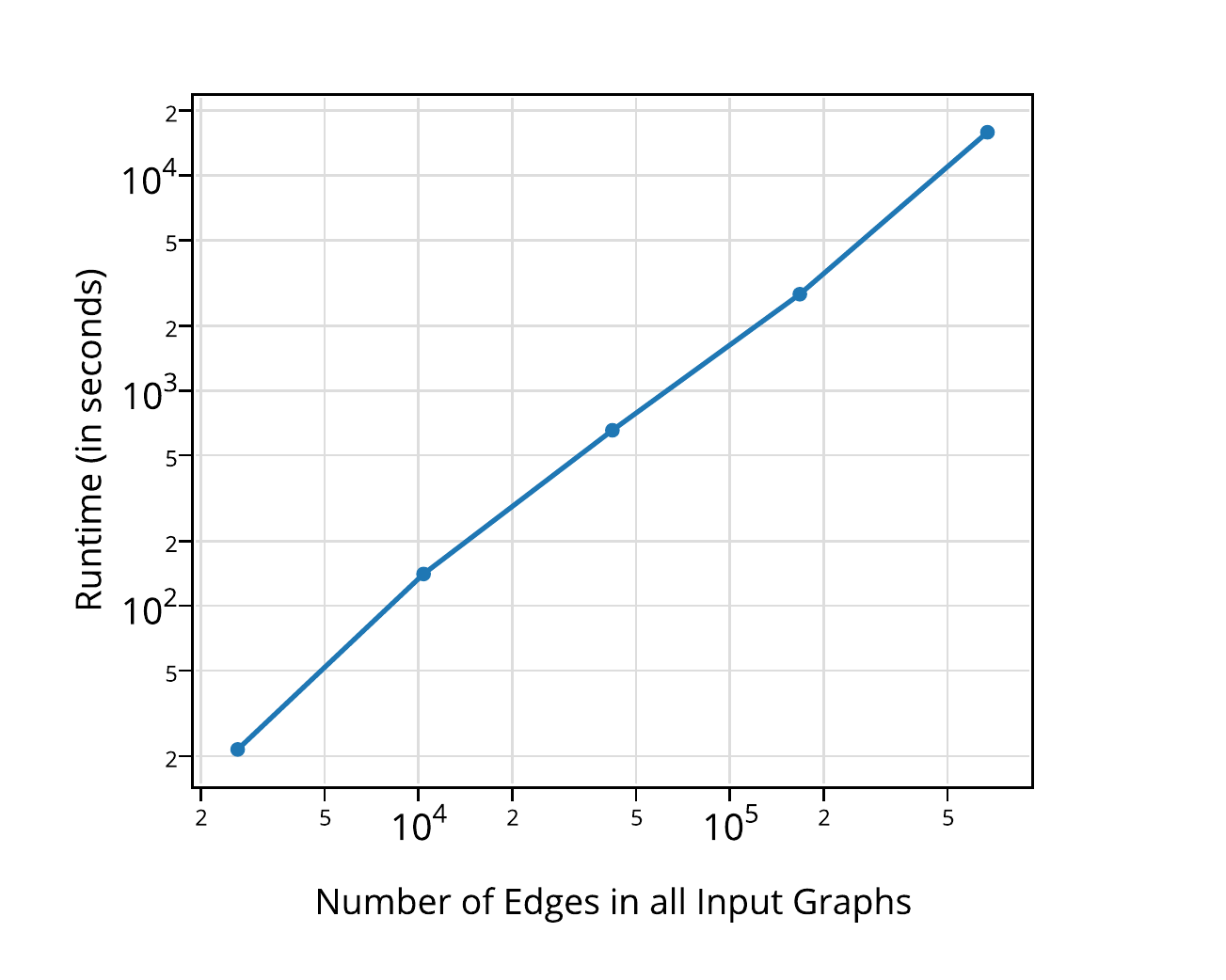}}
\par\end{centering}
\caption{The runtime of LBGA as a function of number of edges on the union of input graphs.} 
\label{fig:scalability-analysis}
\end{figure}

\section{A convergence theorem}
\label{sec:convergence-theorem}

In this section we provide a theorem on the convergence rate of LBGA in the
case that the quality function actually has access to the ground truth
clustering. This is an admittedly unrealistic scenario, but it serves a few
purposes: it is the simplest possible theorem one might hope to prove, it is a
demonstration of how standard MWUA analysis can be used for our problem, and it
makes the theoretical goals and assumptions rigorous.

Let $H_1, \dots, H_N$ be graphs on the same vertex set $V$, and $c$ an unknown
target clustering of $V$. Call $H_i$ \emph{good} for a vertex pair $(u,v)$ if
it agrees with $c$, i.e. $(u,v)$ is both an edge of $H_i$ and $c(u) = c(v)$, or
$(u,v)$ is not an edge of $H_i$ and $c(u)\neq c(v)$. Otherwise call $H_i$
\emph{bad} for $(u,v)$. 

Let $\delta > 0$ be the learning target, so that an edge is considered
``learned'' if the probability $p_i^t$ of choosing it in round $t$ either drops
below $\delta$ or rises above $1 - \delta$. Define the oracle-type quality
function

$$
   q_c(u,v) = 
      \begin{cases}
          1 & \mbox{ if } c(u) = c(v) \\ 
         -1 & \mbox{ if } c(u) \neq c(v) 
      \end{cases}
$$

The theorem is that with access to $q_c$ LBGA will produce the correct graph
after $O(\log(1/\delta))$ rounds. Before we prove this we set up a bit of
notation.

Fix a vertex pair $(u,v)$. Let $w_{u,v,i}$ be the weight of $H_i$ that LBGA
maintains. Define $w_{u,v,\textup{good}} = \sum_{\textup{good } H_i} w_i$ to be
the sum of the weights of the good graphs for $(u,v)$, and
$w_{u,v,\textup{bad}} = \sum_{\textup{bad } H_i} w_i$ the sum of the bad
weights. Call $p_{u,v,\textup{good}}, p_{u,v,\textup{bad}}$ the corresponding
probability weights of picking a good or bad input graph. When $u,v$ are clear
from the context, we will omit them from the subscripts and simply write
$w_{\textup{good}}$, etc. 

\begin{thm}
Suppose LBGA uses the quality funciton $q_c$, a null event $A$, and parameters
$0 < \varepsilon = \nu < 1/2$. Suppose further that for every vertex pair
$(u,v)$ some input graph $H_i$ is good for $(u,v)$. Then for any $0 < \delta <
1$ and for all vertex pairs $(u,v)$, after $O(\log(1 / \delta))$ rounds the
probability $p_{u,v,\textup{bad}} < \delta$.  
\end{thm}

\begin{proof}
Fix a vertex pair $(u,v)$. We add a superscript $t$ to any quantity that
changes over rounds to denote which round of LBGA the quantity refers to. Our
goal is to bound $p_{\textup{bad}}^t$ for $t = O(\log(1/ \delta))$. Denote by
$N_{\textup{bad}} = w_{\textup{bad}}^1, N_{\textup{good}} =
w^1_{\textup{good}}$, the number of bad and good input graphs, respectively. 

The update step in LBGA and the quality function $q_c$ are such that in every
round the weights of good graphs are multiplied by $(1 + \varepsilon)$ and the
weights of the bad graphs are multiplied by $(1 - \varepsilon)$. Hence, 

$$ w^{t+1}_{\textup{good}} = \sum_{H_i \textup{ is good}} w^t_i(1 + \varepsilon) =
w_{\textup{good}}^t (1 + \varepsilon)$$

and likewise for $w_{\textup{bad}}^{t+1}$. By induction, we have 

$$ w^t_{\textup{good}} = N_{\textup{good}} (1 + \varepsilon)^t,  
w^t_{\textup{bad}} = N_{\textup{bad}} (1 - \varepsilon)^t $$

and the value 

$$p_{\textup{bad}}^t =
\frac{N_{\textup{bad}}(1-\varepsilon)^t}{N_{\textup{bad}}(1-\varepsilon)^t +
N_{\textup{good}}(1+\varepsilon)^t} \leq \frac{N_{\textup{bad}}}{(1 +
\varepsilon)^t}$$
holds for all $t, N_{\textup{good}}, N_{\textup{bad}} \in \mathbb{N}$,
$\varepsilon > 0$. If we fix a stopping point $T > 0$ and require
$N_{\textup{bad}} / (1 + \varepsilon)^T \leq \delta$ then we solve for 

$$T \geq \frac{\log (1 / \delta) + \log(N_{\textup{bad}})}{\log (1 + \varepsilon)}
= O(\log(1/ \delta))$$
\end{proof}

Note that this result is deterministic. The oracle-type quality function allows
us to avoid a probabilistic analysis. So the natural generalization is to
weaken the assumption on $q$. For example, one might assume its value
correlates with the truth instead of being absolutely correct, or that its
usefulness is determined i.i.d. in each round. The end goal of this line of
inquiry would characterize how a suitable event $A$ can compensate for
increasingly weak and unreliable $q$, since this is the real-world scenario for
which we posit LBGA is useful.  This is an area for future work.

\section{Conclusions}
\label{sec:conclusion}
We present the Locally Boosted Graph Aggregation framework, a general framework
for learning graph representations with respect to an application. In this
paper, we demonstrate the strength of the framework with the application of
community detection, and we believe the framework can be adapted to other
inference goals in graphs such as link prediction or diffusion estimation. 

Our framework offers a flexible, local weighting and aggregation of different
edge sources in order to better represent the variability of relational
structure observed in real networks. Inspired by concepts in boosting and
bandit learning approaches, LBGA is designed to handle aggregations of noisy
and disparate data sources, therefore marking a departure from methods that
assume overlap and usefulness among all data sources considered. 

LBGA also simplifies the task of designing a graph aggregation algorithm into
designing a principled quality measure $q$ and global event $A$. Doing so
allows us to connect the utility of the graph representation to the application
of interest. As a byproduct, we conjecture that statistics produced by LBGA
provide information about the utility of the graph source, such as the level of
noise. Such information can be used to improve the data collection process and
partially mitigate the effects noise before it propagates to subsequent
analysis. We gave evidence for the resistence of our framework to noise by
running it on datasets with various known levels of noise, and observed the
resulting matching weight distributions. 

Another primary concern for future work is to analyze the utility of our
framework with respect to other graph applications, as well as to present a
more thorough comparison of LBGA with existing multigraph clustering
algorithms. We believe link prediction and label propagation to be ripe
candidates, and there is an established similarity metric for the former of
Adamic and Adar\cite{Adamic01}. Our preliminary experiments have shown
qualitatively different graph structure when using Adamic/Adar in place of
$consistentNO$, but further research is necessary to fully understand the reason
for it, and whether it corresponds to a qualitatively better analysis (in light
of the wealth of literature on metrics for link prediction, we conjecture that
it does).\footnote{There is also a technical barrier to using Adamic/Adar
directly, it is not a [0,1]-valued function.}

Additional directions include a more thorough stability analysis of LBGA,
exploring the modifications we have suggested to improve scalability, and to
prove further theoretical results as described earlier.

\section{Acknowledgements} We thank Vineet Mehta for providing us with the
Enron topic model data, and for his many helpful discussions.

\bibliographystyle{plain}
\bibliography{graphLearning}

\begin{thebibliography}{10}

\bibitem{Enron}
Enron email dataset.
\newblock \url{https://www.cs.cmu.edu/~enron/}.

\bibitem{Adamic01}
L.~Adamic and E.~Adar.
\newblock Friends and neighbors on the web.
\newblock {\em Social Networks}, 25:211--230, 2001.

\bibitem{Aggarwal2011}
C.C. Aggarwal, Y.~Xie, and P.S. Yu.
\newblock Towards community detection in locally heterogeneous networks.
\newblock In {\em SDM}, pages 391--402. SIAM, 2011.

\bibitem{Arora12}
S.~Arora, E.~Hazan, and S.~Kale.
\newblock The multiplicative weights update method: a meta-algorithm and
  applications.
\newblock {\em Theory of Computing}, 8(1):121--164, 2012.

\bibitem{Balcan2006}
M.~Balcan and A.~Blum.
\newblock {\em On a theory of kernels as similarity functions}.
\newblock Mansucript, 2006.

\bibitem{Balcan2008}
M.~Balcan, A.~Blum, and N.~Srebro.
\newblock A theory of learning with similarity functions.
\newblock {\em Machine Learning}, 72(1-2):89--112, 2008.

\bibitem{Berlingerio2011}
M.~Berlingerio, M.~Coscia, and F.~Giannotti.
\newblock Finding redundant and complementary communities in multidimensional
  networks.
\newblock In {\em CIKM}, pages 2181--2184, 2011.

\bibitem{Bubeck12}
S.~Bubeck and N.~Cesa-Bianchi.
\newblock Regret analysis of stochastic and nonstochastic multi-armed bandit
  problems.
\newblock {\em Foundations and Trends in Machine Learning}, 5(1):1--122, 2012.

\bibitem{Bubeck09}
S.~Bubeck, R.~Munos, and G.~Stoltz.
\newblock Pure exploration in multi-armed bandits problems.
\newblock In {\em ALT}, pages 23--37, 2009.

\bibitem{CCK14}
R.~Caceres, K.~Carter, and J.~Kun.
\newblock A boosting approach to learning graph representations.
\newblock In {\em Proc. SIAM Int. Conf. Data Mining, Workshop on Mining
  Networks and Graphs}, 2014.
\newblock To appear.

\bibitem{Caceres2011}
R.~S. Caceres, T.~Y. Berger-Wolf, and R.~Grossman.
\newblock Temporal scale of processes in dynamic networks.
\newblock In {\em ICDM Workshops}, pages 925--932, 2011.

\bibitem{Cai2005}
D.~Cai, Z.~Shao, X.~He, X.~Yan, and Jiawei Han.
\newblock Community mining from multi-relational networks.
\newblock In {\em PKDD}, pages 445--452, 2005.

\bibitem{Danon05}
L.~Danon, J.~Duch, A.~Diaz-Guilera, and A.~Arenas.
\newblock Comparing community structure identification.
\newblock {\em J. Stat. Mech.}, 2005:P09008, 2005.

\bibitem{Dice1945}
L.~R. Dice.
\newblock Measures of the amount of ecologic association between species.
\newblock {\em Ecology}, 26(3):297--302, July 1945.

\bibitem{RealityMining}
N.~Eagle and A.~Pentland.
\newblock Reality mining: sensing complex social systems.
\newblock {\em Personal and Ubiquitous Computing}, 10(4):255--268, 2006.

\bibitem{Erdos60}
P.~Erd\"{o}s and A.~R\'{e}nyi.
\newblock {On random graphs, I}.
\newblock {\em Publicationes Mathematicae (Debrecen)}, 6:290--297, 1959.

\bibitem{Fortunato07}
S.~Fortunato and M.~Barth\'{e}l\'{e}my.
\newblock Resolution limit in community detection.
\newblock {\em Proceedings of the National Academy of Sciences}, 104(1):36--41,
  2007.

\bibitem{Gallagher2008}
B.~Gallagher, H.~Tong, T.~Eliassi-Rad, and C.~Faloutsos.
\newblock Using ghost edges for classification in sparsely labeled networks.
\newblock In {\em Proc. (14th) ACM SIGKDD Inter. Conf. on Knowledge Discovery
  and Data Mining}, pages 256--264. ACM, 2008.

\bibitem{Getoor2005}
L.~Getoor and C.~P Diehl.
\newblock Link mining: a survey.
\newblock {\em ACM SIGKDD Explorations Newsletter}, 7(2):3--12, 2005.

\bibitem{Gilbert2009}
E.~Gilbert and K.~Karahalios.
\newblock Predicting tie strength with social media.
\newblock In {\em Proc. of SIGCHI Conf. on Human Factors in Computing Systems},
  CHI '09, pages 211--220, New York, NY, USA, 2009. ACM.

\bibitem{Gleich2012}
D.~F. Gleich and C.~Seshadhri.
\newblock Vertex neighborhoods, low conductance cuts, and good seeds for local
  community methods.
\newblock In {\em Proceedings of the 18th ACM SIGKDD International Conference
  on Knowledge Discovery and Data Mining}, KDD '12, pages 597--605, New York,
  NY, USA, 2012. ACM.

\bibitem{Jaccard1912}
P.~Jaccard.
\newblock The distribution of the flora in the alpine zone.
\newblock {\em New Phytologist}, 11(2):37--50, February 1912.

\bibitem{EnronConf}
B.~Klimt and Y.~Yang.
\newblock The enron corpus: A new dataset for email classification research.
\newblock In {\em ECML}, pages 217--226, 2004.

\bibitem{Leskovec2008}
J.~Leskovec, K.~J. Lang, A.~Dasgupta, and M.~W. Mahoney.
\newblock Statistical properties of community structure in large social and
  information networks.
\newblock In {\em Proc. (17th) Inter. Conf. on World Wide Web}, WWW '08, pages
  695--704, New York, NY, USA, 2008. ACM.

\bibitem{Ley02}
M.~Ley.
\newblock The dblp computer science bibliography: Evolution, research issues,
  perspectives.
\newblock In {\em SPIRE}, pages 1--10, 2002.

\bibitem{Mccallum05}
A.~McCallum, A.~Corrada-Emmanuel, and X.~Wang.
\newblock The author-recipient-topic model for topic and role discovery in
  social networks, with application to enron and academic email.
\newblock In {\em Workshop on Link Analysis, Counterterrorism and Security},
  pages 33--44, Newport Beach, CA, 2005.

\bibitem{mallet}
A.~K. McCallum.
\newblock {MALLET: A Machine Learning for Language Toolkit}.
\newblock http://mallet.cs.umass.edu, 2002.

\bibitem{Miller2014}
B.~A. Miller and N.~Arcolano.
\newblock Spectral subgraph detection with corrupt observations.
\newblock In {\em Proc. IEEE Int. Conf. Acoust., Speech and Signal Process.},
  2014.
\newblock To appear.

\bibitem{Mucha2010}
P.~J. Mucha, T.~Richardson, K.~Macon, M.~A. Porter, and J.~P. Onnela.
\newblock Community structure in time-dependent, multiscale, and multiplex
  networks.
\newblock {\em Science}, 328(5980):876--878, 2010.

\bibitem{nadakuditi2012}
R.~R. Nadakuditi and M.~E.~J. Newman.
\newblock Graph spectra and the detectability of community structure in
  networks.
\newblock {\em CoRR}, abs/1205.1813, 2012.

\bibitem{Neville2005}
J.~Neville and D.~Jensen.
\newblock Leveraging relational autocorrelation with latent group models.
\newblock In {\em Proc. (4th) Inter. Workshop on Multi-relational Mining},
  pages 49--55. ACM, 2005.

\bibitem{Newman06}
M.~E.~J. Newman.
\newblock Modularity and community structure in networks.
\newblock {\em Proceedings of the National Academy of Sciences of the United
  States of America}, 103(23):8577--8696, 2006.

\bibitem{Papalexakis2013}
E.~E. Papalexakis, L.~Akoglu, and D.~Ience.
\newblock Do more views of a graph help? community detection and clustering in
  multi-graphs.
\newblock In {\em FUSION}, pages 899--905, 2013.

\bibitem{Walktrap}
P.~Pons and M.~Latapy.
\newblock Computing communities in large networks using random walks.
\newblock {\em Journal of Graph Algorithms and Applications}, 10(2):191--218,
  2006.

\bibitem{Rocklin2013}
M.~Rocklin and A.~Pinar.
\newblock On clustering on graphs with multiple edge types.
\newblock {\em Internet Mathematics}, 9(1):82--112, 2013.

\bibitem{Rossi2012}
R.~A. Rossi, L.~McDowell, D.~W. Aha, and J.~Neville.
\newblock Transforming graph data for statistical relational learning.
\newblock {\em J. Artif. Intell. Res. (JAIR)}, 45:363--441, 2012.

\bibitem{Schapire90}
R.~E. Schapire.
\newblock The strength of weak learnability.
\newblock {\em Machine Learning}, 5:197--227, 1990.

\bibitem{Tang2012}
L.~Tang, X.~Wang, and H.~Liu.
\newblock Community detection via heterogeneous interaction analysis.
\newblock {\em Data Min. Knowl. Discov.}, 25(1):1--33, 2012.

\bibitem{Tang2009}
W.~Tang, Z.~Lu, and I.~S. Dhillon.
\newblock Clustering with multiple graphs.
\newblock In {\em Proc. (2009) IEEE Int. Conf. on Data Mining}, ICDM '09, pages
  1016--1021, Washington, DC, USA, 2009. IEEE Computer Society.

\bibitem{Wang87}
Y.~Wang and G.~Wong.
\newblock {Stochastic Block Models for Directed Graphs}.
\newblock {\em Journal of the American Statistical Association}, 82(397):8--19,
  1987.

\bibitem{Xiang2010}
R.~Xiang, J.~Neville, and M.~Rogati.
\newblock Modeling relationship strength in online social networks.
\newblock In {\em Proc. (19th) Int. Conf. on World Wide Web}, WWW '10, pages
  981--990, New York, NY, USA, 2010. ACM.

\end{thebibliography}

\end{document}